\documentclass[letterpaper, 10 pt, conference]{ieeeconf}
\usepackage{amssymb}
\usepackage{amsmath,amsfonts}
\usepackage{array}
\usepackage[caption=false,font=normalsize,labelfont=sf,textfont=sf]{subfig}
\usepackage{textcomp}
\usepackage{stfloats}
\usepackage{url}
\usepackage{verbatim}
\usepackage{graphicx}
\usepackage{cite} 
\usepackage{xcolor}
\usepackage{hyperref}
\usepackage{float}
\usepackage{booktabs}
\usepackage[ruled]{algorithm2e}
\usepackage{algpseudocode}
\usepackage{dsfont}
\usepackage{etoolbox}
\usepackage{microtype}
\usepackage{bbm}
\providetoggle{anon}
\newtheorem{theorem}{Theorem}
\newtheorem{lemma}{Lemma}
\newtheorem{assumption}{Assumption}
\newtheorem{corollary}{Corollary}
\SetKw{Params}{Params}
\DontPrintSemicolon

\IEEEoverridecommandlockouts
\settoggle{anon}{false}

\begin{document}
\title{EB-MBD: Emerging-Barrier Model-Based Diffusion for\\ Safe Trajectory Optimization in Highly Constrained Environments}

\author{\iftoggle{anon}{Anonymous}{
Raghav Mishra and Ian R. Manchester 
\thanks{This research was funded by the Australian Research Council through the ARC Research Hub in Intelligent Robotic Systems for Real-Time Asset Management (IH210100030). All authors are with the Australian Robotic Inspection and Asset Management (ARIAM) Hub, the Australian Centre for Robotics, and the School of Aerospace, Mechanical and Mechatronic Engineering, University of Sydney}
}
}

\markboth{}{}


\maketitle

\begin{abstract}
 We propose enforcing constraints on Model-Based Diffusion by introducing emerging barrier functions inspired by interior point methods. We demonstrate that the standard Model-Based Diffusion algorithm can lead to catastrophic performance degradation in highly constrained environments, even on simple 2D systems due to sample inefficiency in the Monte Carlo approximation of the score function. We introduce Emerging-Barrier Model-Based Diffusion (EB-MBD) which uses progressively introduced barrier constraints to avoid these problems, significantly improving solution quality, without expensive projection operations such as projections. We analyze the sampling liveliness of samples at each iteration to inform barrier parameter scheduling choice. We demonstrate results for 2D collision avoidance and a 3D underwater manipulator system and show that our method achieves lower cost solutions than Model-Based Diffusion, and requires orders of magnitude less computation time than projection based methods. 
\end{abstract}



\section{Introduction}
\label{sec:intro}
Dynamic motion planning problems are often formulated as constrained trajectory optimization problems of the form
\begin{gather*}
    \min_{\xi_{0:T+1}, u_{0:T}} \left\{ J(\xi_{0:T}, u_{0:T}) = \ell_T(\xi_{T+1}) + \sum_{t=0}^{K} \ell(\xi_t, u_t)\right\} \\
    \text{s.t. } \xi_{t+1} = f(\xi_t, u_t),\quad  g(\xi_{0:T}, u_{0:T}) \geq 0.
\end{gather*}
These problems can be solved online using a numerical non-linear programming (NLP) solver, which performs optimization on an initial guess until convergence to locally optimal solutions, often using first or second order numerical algorithms such as Newton-Raphson, BFGS, etc.

While trajectory optimization has been successful for many systems, it struggles to optimize in the non-convex, non-smooth landscapes introduced by systems such as manipulators and legged robots\cite{suh_bundled_2022}. Feasible and performant trajectory optimization in these domains relies on hand-crafted initialization and contact schedules. Additionally, not all dynamic models or simulators are differentiable and thus can not provide derivative information for optimization. Recently, the robotics community has embraced learning-based methods such as imitation learning and reinforcement learning due to their empirically strong performance for problems on which trajectory optimization struggles\cite{chi_diffusion_2023}.

Diffusion models, which have gained popularity for their ability to generate realistic images\cite{ho_denoising_2020, sohl-dickstein_deep_2015}, have demonstrated impressive performance in representing learned control policies~\cite{chi_diffusion_2023, janner_planning_2022}. However, the model-free nature of standard diffusion models discounts the knowledge we have about system dynamics and task objectives, which has led to the development of diffusion algorithms that incorporate model knowledge for trajectory optimization~\cite{pan_model-based_2024, carvalho_motion_2024, xue_full-order_2024}. These approaches can augment learning-based diffusion with model knowledge, or even run diffusion without any learning.

Compared to optimization approaches, these approaches provide multi-modal trajectory sampling, less susceptibility to local minima, the ability to incorporate learning~\cite{pan_model-based_2024}, and optimize through non-smooth dynamics such as contact~\cite{suh_bundled_2022}. Compared to learning-only diffusion models, they can incorporate knowledge of dynamics and objectives, benefiting generalization in new contexts, and offer flexible inference without retraining or gathering new data. Model-based diffusion (MBD) \cite{pan_model-based_2024} uses Monte Carlo approximations of the Stein score function to run a learning-free reverse diffusion process, providing a zeroth-order sampling-based alternative to traditional gradient-based optimization.

\begin{figure}[t]
    \centering
    \includegraphics[width=0.95\linewidth]{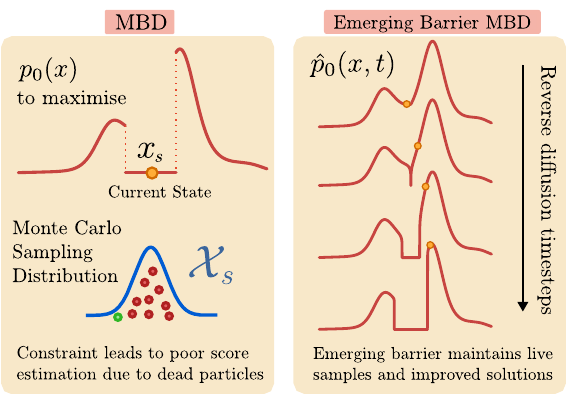}
    \caption{Our proposed method improves the performance of Model Based Diffusion by augmenting the underlying target distribution over the reverse process with a log barrier cost}
    \label{fig:main_figure}
\end{figure}

Motion planning problems often have constraints such as collision avoidance or dynamic safety, which gradient-based methods have mature ways to incorporate. In the MBD framework, infeasible solutions result in zero-probability regions on the support of the distribution. Although various methods for enforcing constraints on learning-based diffusion have been proposed\cite{khalafi_constrained_2024, fishman_diffusion_2024, romer_diffusion_2024, christopher_constrained_2024}, they do not translate to MBD due to its reliance on Monte Carlo sampling which is also affected by the constraints. When infeasible solutions cover a large part of the solution space, the score estimate's value is dominated by rare events on the Monte Carlo proposal distribution, degrading performance.

The contributions of this paper are:
\begin{itemize}
    \item \textbf{We show that MBD can suffer catastrophic performance degradation} in highly constrained state spaces due to poor score estimation.
    \item \textbf{We propose Emerging Barrier MBD}, which augments MBD with a time-varying barrier for constraint guidance towards higher-quality and guaranteed feasibility.
    \item \textbf{We analyse the sampling statistics of the EB-MBD process} and study the design trade-offs of the barrier hyperparameters.
    \item \textbf{We compare EB-MBD to MBD and projection-based constrained diffusion methods} in simulation experiments and show that it provides better performance at a fraction of the computational cost.
\end{itemize}

\section{Related Work}
\label{sec:related}
\subsection{Diffusion Models}
Diffusion Models \cite{sohl-dickstein_deep_2015, ho_denoising_2020} simulate a noising diffusion process until a time where the distribution is stationary and easy to sample from, and train a neural network to learn the score function (or a discretized equivalent) of the process. This allows generative modelling by sampling from the simpler stationary distribution and running the reverse processes to obtain samples from the target distribution. 

Diffusion models have shown great capability in sampling from complex modalities and distributions such as images, videos, and text, but have been used in robotics for representing policies and planners via imitation learning \cite{chi_diffusion_2023} and reinforcement learning. Compared to other types of policy representations, diffusion models excel at representing multi-modal distributions~\cite{janner_planning_2022, chi_diffusion_2023}. For robotics, they also enable simple conditional guidance mechanisms to allow task-specific sampling.

While various attempts have been made to run diffusion models with constraints, they usually involve modifying the training process \cite{khalafi_constrained_2024,fishman_diffusion_2024}. \cite{fishman_diffusion_2024} uses log barriers for constrained diffusion to modify the diffusion step size through the Riemannian metric induced by the barrier's Hessian. However, this method is limited to convex constraints.  Alternatively, they may try to enforce constraints on a pre-trained model during the denoising steps \cite{christopher_constrained_2024, xiao_safediffuser_2023, romer_diffusion_2024}. Many of these approaches only focus on learned diffusion models and do not translate to Monte Carlo score estimation which brings additional challenges due to how they are affected by constraints. Additionally, many use computationally expensive operations like projections onto non-convex constraint sets \cite{romer_diffusion_2024, christopher_constrained_2024} or optimization problems \cite{xiao_safediffuser_2023} which are passed to numerical NLP solvers whose solution quality and computation time can depend heavily on the initial guess and cost functions.   
\subsection{Trajectory optimization}
Trajectory optimization poses motion planning problems as non-linear programming problems. Trajectory optimization uses the rich theory of numerical optimization with the ability to specify and enforce constraints through various methods such as penalty methods, barrier methods, augmented Lagrangian methods, projected gradient descent, etc. Most relevant to this work, interior point methods use log barriers that enforce infinite cost at the boundary of the constraint and involve repeatedly solving the barrier-augmented optimization problem but with reducing the weight or ``hardness" of the barrier until convergence to the true solution. 

However, gradient-based optimization is difficult to parallelize, cannot deal with non-smooth objectives, and suffers from local minima. Gradient-free trajectory optimization has been studied to address some problems with gradient-based optimization. Sampling-based optimal control algorithms such as STOMP \cite{kalakrishnan_stomp_2011}, MPPI \cite{williams_aggressive_2016} and its variants \cite{lefebvre_path_2019, theodorou_iterative_2011, theodorou_reinforcement_2010} have shown strong ability to deal with complex non-linear systems. 

In recent literature, many works have started to augment diffusion algorithms with ideas from trajectory optimization. Various works have implemented this through guidance from model-based cost functions\cite{carvalho_motion_2024}, or used Monte Carlo sampling to run diffusion algorithms \cite{pan_model-based_2024, xue_full-order_2024} in a similar vein to sampling-based trajectory optimization.

\section{Technical Background And Motivation}
\subsection{Diffusion Models}
Diffusion models approach the problem of sampling from a target distribution with density $p(x)$ by setting it as $p_0(x)$, the initial distribution at time $t=0$, for a diffusion stochastic process which adds noise to the target distribution until it approaches a stationary distribution that is easy to sample from. By sampling from this simple distribution and running the reverse process, we can recover samples from the original target distribution. For example, Denoising Diffusion Probabilistic Models (DDPM)\cite{ho_denoising_2020} has the following forward process which adds corrupting noise to samples from the target distribution during training
\begin{equation}    
    x_{s+1} = \sqrt{1 - \beta_s} \, x_s + \beta_s z_s\label{eq:ddpm_forward},
\end{equation}
and runs the reverse process to produce samples from the target distribution during inference
\begin{equation}
    x_{s-1} = \frac{1}{\sqrt{\alpha_s}}\left[x_s + \frac{1-\alpha_s}{\sqrt{1-\bar{\alpha}_s}}\epsilon_{\theta}(x_s,s) \right] + \varsigma_s z_s \label{eq:ddpm},
\end{equation}
where $\beta$, $\alpha$, $\bar \alpha$ and $\varsigma$ are parameters that depend on the noise schedule, which is a hyperparameter of the algorithm, and $z_t$ is drawn from a standard Gaussian. The term $\epsilon_\theta$ represents the mean of a denoising Gaussian term added in the reverse process, parametrized by a neural network with parameters $\theta$. In a ``score-based" framework, this discretized reverse process can also be written as
\begin{equation}
    x_{s-1} = \frac{1}{\sqrt{\alpha}_s}\left[x_s + (1-\bar{\alpha}_s) \nabla \log p_s(x_s) \right] + \varsigma_s z_s \label{eq:ddpm2},
\end{equation}
where the $\nabla \log p_s (x) $ is known as the Stein score and is generally approximated by a neural network. The forward process has a standard Gaussian distribution $\mathcal{N}(0, I_N)$ as its stationary distribution, which can be tractably sampled. In general, neither the target distribution nor the score function is available for a given problem. Therefore, the score function (or denoising mean) is learned by taking available samples from the target  distribution, running the forward process, and training a neural network to minimize an evidence lower bound or score matching loss. 

\subsection{Model-based Diffusion}
Given an optimization problem to minimize $J(x)$ over decision variable $x$, we can construct a Boltzmann-Gibbs distribution,
\begin{equation}    
        p(x) \propto \exp\left(-\frac{1}{\lambda}J(x)\right) \label{eq:boltzmann},
\end{equation}
where the density $p(x)$ is higher where the cost is lower and temperature, $\lambda$, is a constant parameter, decreasing which concentrates the mass closer to the minima of the function. This turns optimization of $J(x)$ into sampling from unnormalized densities.

Model Based Diffusion performs sampling by running the underlying process behind the DDPM algorithm, which requires having access to the score function, $\nabla  \log p_s(x_s)$. In the learning-free setting, since we do not directly have access to the score function but do have access to an unnormalized $p_0(x)$, MBD\cite{pan_model-based_2024} uses Monte Carlo sampling to estimate the score
\begin{equation}
    \nabla \log p_s(x_s) \approx - \frac{x_s}{1-\bar \alpha} + \frac{\sqrt{\bar\alpha}}{1-\bar{\alpha}} \left(\frac{\sum_{\hat x_i\in\mathcal{X}_s}\hat x_i\, p_0(\hat x_i)}{\sum_{\hat x_i\in\mathcal{X}_s}p_0(\hat x_i)} \right), \label{eq:mcscore}
\end{equation}
where $\mathcal{X}_s=\{\hat x_i\}_0^N$ and $\hat x_i$ are realizations drawn from $X_s\sim\mathcal{N}(\frac{x_s}{\sqrt{\bar \alpha_{s-1}}}, \sigma_s^2 I)$ and $\sigma_s^2 = \frac{1}{\sqrt{\bar \alpha_{s-1}}}-1$. This is similar to zeroth order optimization methods such as the cross-entropy method (CEM)\cite{rubinstein_cross-entropy_1999}, random search\cite{mania_simple_2018}, and model predictive path integral control\cite{williams_aggressive_2016}. For a dynamic system, evaluating $p_0(\tau)$ for the trajectory, $\tau$, the trajectory can be parametrized by the actions $u_{0:T}$ and rolled out to calculate $J(\tau)$.

However, this Monte Carlo approach poses problems in a constrained environment where a large part of the solution space violates constraints. We demonstrate that its performance catastrophically degrades as constraints become restrictive even on simple 2D systems.  If the problem is heavily constrained, the target density $p(x)=0$ in much of the solution space and samples $\hat x_i$ are likely to have no contribution, leading to ``dead" samples with no information. Figure \ref{fig:trajectory_comparison} shows how increasing obstacle radius in an existing obstacle avoidance problem causes failure, which MBD otherwise performs well on.

\begin{figure}[t]
    \centering
    \subfloat{
        \includegraphics[width=0.2\textwidth]{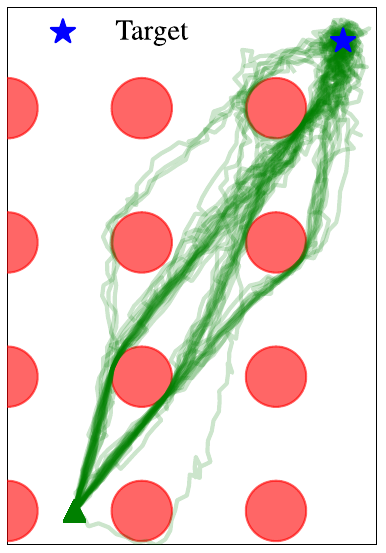}
        \label{fig:low_constrained}
    }
    \subfloat{
        \includegraphics[width=0.2\textwidth]{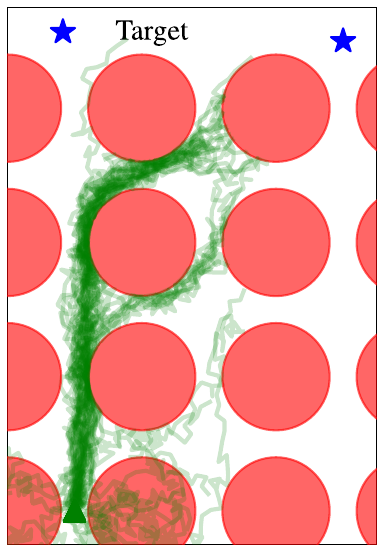}
        \label{fig:high_constrained}
    }
    \caption{Increasing obstacle size in a 2D obstacle avoidance problem leads to catastrophic degradation in performance for MBD as the sampling efficiency of the score estimate suffers}
    \label{fig:trajectory_comparison}
\end{figure}

\section{Emerging Barrier Model-Based Diffusion}
Our motivating context is robotic motion planning where it is common to encounter highly constrained optimization problems. However, our contribution is more generally applicable as an optimization-by-sampling algorithm.
\subsection{Problem Statement}
We aim to solve the constrained optimization problem 
\begin{equation}
    \min_x J(x) \quad 
    \text{ s.t. } \quad g(x) \geq 0 .
\end{equation}
We approximate this problem to one of sampling from distribution
\begin{align*}
    p(x) &\propto \exp\left(-\frac{1}{\lambda} J(x)\right) \mathds{1}_{g}(x),
\end{align*}
where $\mathds{1}_g(x)$ is an indicator function that is zero on the constraint violating set and one on the feasible set. Sampling from $p(x)$ provides high quality approximate solutions to the original problem. We perform sampling by setting the target distribution at $s=0$ as $p_0(x) :=p(x)$, and running a DDPM process as in equation \eqref{eq:ddpm2} using a score approximation as in equation \eqref{eq:mcscore}. MBD\cite{pan_model-based_2024} uses a non-stochastic form of DDPM where $\varsigma_s=0$ and the only stochasticity comes from Monte Carlo sampling and we follow that convention.

\subsection{Motion Planning}
To apply our method to constrained motion planning, we work with the discrete time non-linear systems of the form
\begin{equation*}
\xi_{t+1} = f(\xi_t, u_t),
\end{equation*}
where $\xi \in \mathbb{R}^n$, $u\in \mathbb{R}^m$ and $f(\cdot) : \mathbb{R}^n \times \mathbb{R}^m\rightarrow \mathbb{R}^n$ represent general nonlinear dynamics. 

We parameterize a trajectory, $\tau$, of the system via the actions taken, $\tau = \{u_{0:T}\} \in \mathbb{R}^{m\times T}$, where $T$ is the control time horizon length, which allows sampling feasible trajectories without enforcing equality constraints from the dynamics. We assume the availability of an oracle (such as a simulator) that can be queried with actions to run rollouts to find $\xi_{0:T+1}$. 

We specify a planning task through a cost function that we seek to minimize
\begin{gather*}
    \min_{\xi_{0:T+1}, u_{0:T}} J(\xi_{0:T+1}, u_{0:T}),
\end{gather*}
where $J: \mathbb{R}^{n\times (T+1)}\times \mathbb{R}^{m\times T} \rightarrow \mathbb{R}$ is the total trajectory cost. 
We specify the constraint as an inequality $g(\tau) \geq 0$ where $g(\tau) : \mathbb{R}^{m\times T} \rightarrow \mathbb{R}$. For example, for a collision avoidance problem, this might be a function mapping to the distance to the closest obstacle.  Note that the diffusion process time index, $s \in [0\dots S]$, in $p_s(\tau)$ corresponds iterations of our diffusion process and is different from, $t\in[0\dots T+1]$, the index for the temporal component of our trajectory.

\subsection{Emerging Barriers}
Our proposed solution involves modifying the target distribution  with a time-varying barrier that emerges over the process. Every iteration of the diffusion process, we run a DDPM style update as in equation \eqref{eq:ddpm2}. However, inspired by interior point methods for constrained optimization, we introduce a barrier cost function,
\begin{equation*}
    b(x, s) = \begin{cases}
        -\mu_s \log (g(x) + c_s), & g(x) + c_s \geq 0\\
        \infty, & g(x) + c_s < 0
    \end{cases}\quad,
\end{equation*}
where $\mu$ controls the ``hardness" of the barrier term, $c_s$ is a positive time-varying offset that acts as a constraint relaxation term. We augment the problem cost and target distribution with $b(x, s)$ by defining a time-varying ``target distribution", $\hat p_0(x,s)$, that varies over the diffusion time horizon
\begin{equation*}
    \hat p_0(x,s) = \exp\left(-\hat J(x, s)\right)=  \exp\left(-\frac{1}{\lambda}J(x) - b(x, s) \right),
\end{equation*}
where the barrier term helps approximate $\mathds{1}_g$. Intuitively, when the relaxed constraint $g(x) + c_s \leq 0$, the cost is infinite and thus $\hat p_0(x)=0$. When $c_s > \inf g(x)$, we have no ``dead" samples. This allows us to start off with an unconstrained solution space with tightening constraints guided by the barrier. The final algorithm which we call Emerging Barrier MBD can be seen in Algorithm \ref{barriermbd}.

\begin{figure}[t]
    \centering    \includegraphics[width=0.9\linewidth]{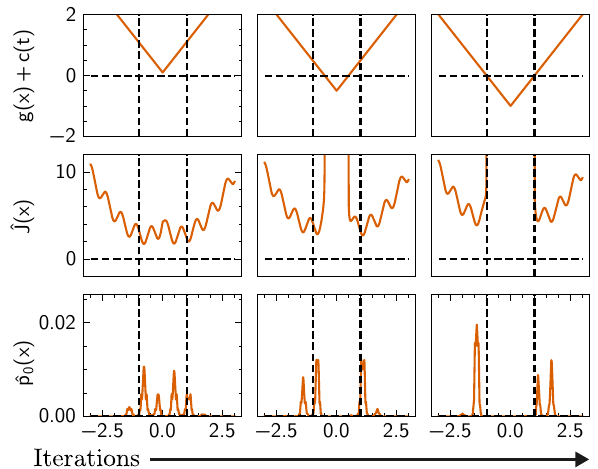}
     \caption{Evolution of relaxed constraint $g(x) + c_s$, the time-varying cost function and the associated normalized density, for the constraint $|x| \leq 1$}
    \label{fig:mountain_plot}
\end{figure}

The barrier cost in the feasible region can be seen as a force that encourages distance away from the constraint boundary and moves with a moving constraint boundary. The emerging barrier can be thought of as complementary to the diffusion itself, which also assigns non-zero probability to zero probability regions in the target distribution by Gaussian smoothing.

For a meaningful emerging barrier schedule, $c_s$ decreases over the reverse process, to enforce constraints as in Figure \ref{fig:mountain_plot}. $\mu_s$ can either decrease to $0$ or be kept at a tuned constant value (e.g. for a obstacle avoidance problem high $\mu_s$ encourages distance from obstacles). At $s=0$, $c_s = 0$ and $\mu_s$ is small (such that $J(x) \gg b(x)$ for $g(x) < 0$). We note that if $\mu, c\rightarrow(0,0)$ as $s\rightarrow0$, the target probability distribution converges pointwise to our true target distribution,
\begin{gather*}
    \lim_{s \rightarrow 0} \left[ \exp\left(\frac{1}{\lambda}J(x) + b(x, s) \right)\right]  = \exp\left(\frac{1}{\lambda}J(x)\right) \mathds{1}_g(x).
\end{gather*}

Additionally, although the cost can become $\infty$, since the equations \eqref{eq:ddpm2} and \eqref{eq:mcscore} directly use the probability density, which is $0$ for constraint-breaking samples, the operations taken every step are still well defined. 

\begin{algorithm}
    \caption{Emerging Barrier MBD}
    \label{barriermbd}
    \KwIn{Noise schedule $\beta_s$, Barrier schedule $(\mu_s, c_s)$, Diffusion steps $S$, Cost function $J(x)$, Constraint $g(x)$, Temperature $\lambda$}
    \tcp{Compute scheduling variables}
    $\alpha_s \gets 1 - \beta_s$ 
    
    $\bar \alpha_s \gets \prod_{i=0}^s \alpha_i$\;

    \tcp{Sample initial diffusion state}
    $x_S \sim \mathcal{N}(0, I)$ 
    
    \For{$s \gets S$ \KwTo $0$}{
        \tcp{Sample around current state}

        $\mathcal{X}_s\sim\mathcal{N}(\frac{x_s}{\sqrt{\bar \alpha_{s-1}}}, \frac{I}{\sqrt{\bar \alpha_{s-1}}}- I)$ 
        
        \tcp{Update barrier parameters}
        
        $\hat p_0(x) := \exp\left(\frac{1}{\lambda}J(x)-\mu_s \log (g(x) +c_s)\right)$
        
        \tcp{Compute score approximation} 
    
        $\gamma \gets - \frac{x_s}{1-\bar \alpha_s} + \frac{\sqrt{\bar\alpha_s}}{1-\bar{\alpha_s}} \left(\frac{\sum_{\hat x_i\in\mathcal{X}_s}\hat x_i\, \hat p_0(\hat x_i)}{\sum_{\hat x_i\in\mathcal{X}_s}\hat p_0(\hat x_i)} \right)$  

        \tcp{Run reverse diffusion step}
        
        $x_{s-1} \gets \frac{1}{\sqrt{\alpha_s}}\left[x_s + (1-\bar{\alpha}_s)\gamma\right]$   
    }
\end{algorithm}
EB-MBD has schedules of the barrier offset $c_s$ and barrier softness $\mu_s$ as hyperparameters of the algorithm. The key challenge in tuning EB-MBD is maintaining ``alive" samples throughout the process despite the constraint tightening over the process.

We may interpret the behavior of emerging barriers as occurring in two regimes which we refer to as the global regime and the local regime. At the beginning, model-based diffusion operates in the global regime, with high sampling noise and relaxed constraints. In this regime, EB-MBD explores many local minima e.g. attempting various modes of reaching the target in Figure \ref{fig:low_constrained}. Towards the end of the diffusion process, EB-MBD is in the local regime, where the sampling noise is small and the iterations closely approximate gradient ascent on the unmodified $p_0(x)$ -- this corresponds to refining a single trajectory in Figure \ref{fig:low_constrained}.
\subsection{Barrier analysis}\label{sec:barrieranalysis}

As $c_0=0$, any schedule that reduces $c_s$ slowly early in the process must speed up proportionally later in the process and vice versa. A constraint that rapidly decreases $c_s$ in the beginning produces a larger number of ``dead" infeasible samples, suffering from lack of gradient information, similar to MBD. A slow progression has higher quality gradient information but towards the end of the process, as $c_s$ is required to converge to $0$ quickly, the process may struggle with dead samples. This can mean infeasible solutions as the iterates cannot keep up with the progression of the barrier and remain dead permanently. This exposes a trade-off in the offset schedule as a design decision.

With some assumptions, we can analyze the worst-case behavior in the local regime of the EB-MBD process to lower-bound the probability of dead samples. 

\begin{assumption} 
The constraint function $g(x)$ is a linear signed distance function $g(x) = w^Tx + b$ with $||\nabla g(x)||=||w||=1$, where $w$ has the same dimensions as $x$ and $b$ is scalar  \label{assumption:linearisation}
\end{assumption}
We justify this as the first order Taylor series of $g(x)$ which is a valid approximation in the limit towards the end of the process, as we will be sampling $g(x_s+\varepsilon)$, where $\varepsilon$ is a perturbation with very small variance. We note that all SDFs are $1$-Lipschitz and almost everywhere $||\nabla g(x)||=1$.
\begin{lemma}
    Suppose Assumption \ref{assumption:linearisation} holds, then a given $x_{s}$, the probability that a random sample of $X_s$ is alive is lower bounded by 
    \begin{gather*}
        \mathbb{P}(X_s \text{ alive})  \geq \Phi\left(\frac{1}{\sigma_s}\left(g(x_{s})  + c_s - \frac{1- \sqrt{\bar \alpha_{s-1}}}{\sqrt{\bar \alpha_{s-1}}}||x_{s}||\right) \right),
    \end{gather*}\label{lemma:Pnextgivenlast}
where $\sigma^2$ is the sampling variance of the MBD process, and $\Phi(\cdot)$ is the univariate Gaussian cumulative distribution function.
\end{lemma}
The proof for this lemma is provided in Appendix \ref{appendix:lemma}. A sample $x$ is alive if  $g(x) + c_s \geq 0$, for which we get the probability by
\begin{equation*}
    \mathbb{P}(g(X_s) + c_s \geq 0| x_{s}).
\end{equation*}
 MBD samples from $X_s \sim\mathcal{N}(\frac{x_{s}}{\sqrt{\bar \alpha_{s-1}}}, \sigma_s^2 I)$. We can arrive at Lemma \ref{lemma:Pnextgivenlast} by expanding and applying Assumption \ref{assumption:linearisation}. We can interpret this as the probability increasing with $g(x_s)$ -- which is a distance from constraint boundary for an SDF --, our offset $c_s$, and with increasing sampling variance, $\sigma_s$. 

Lemma \ref{lemma:Pnextgivenlast} provides some insight but it depends critically on the location of $x_{s}$. For an active constraint, the barrier provides a repulsive force that pushes iterates away from the constraint while the objective function does the opposite. This forms a ``boundary layer" where the diffusion state will tend to be attracted. To provide further insight we next assume that $x_s$ is at the local minimum of $J(x,s+1)$ at this boundary layer. 
\begin{assumption}
    $\nabla J(x)$ is $M_J$-Lipschitz. \label{assumption:LsmoothJ}
\end{assumption} 
This is a common assumption in optimization and is met by common costs such as quadratic functions. 
\begin{assumption}
    The local minima, $x^\star$ of $\hat J(x, s)$ is bounded $||x^\star||\leq R$
\end{assumption}
A reasonable bound for the target distribution $X_0$ is often known, e.g. due to actuator limits, and the source distribution's optimum is bounded as $X_S \sim N(0, I_d)$. Thus, we are assuming the diffusion process that interpolates between the two distributions keeps bounded minima.
\begin{corollary}
Assumption \ref{assumption:LsmoothJ} implies that $J(x)$ is $L_J$-Lipschitz continuous within the domain $\{x : ||x|| \leq R\}$,  where $L_J = ||\nabla J(0)|| +M_J R$
\end{corollary}

\begin{assumption}
    $x_s$ is located at the local minimum of the $\hat J(x, s+1)$ where $\nabla J(x^\star_{s+1}) + \nabla b(x^\star_{s+1}, s+1) = 0$. \label{assumption:optimalIter}
\end{assumption}
This can be justified as MBD's optimization process on annealed the target density occurring on a faster time scale than the change in the barrier. The tight tracking of the local minima by MBD is visible experimentally in the original work\cite{pan_model-based_2024}. We refer to \cite{starnes_gaussian_2024} for a convergence analysis of similar algorithms.
\begin{theorem}
 Suppose assumptions \ref{assumption:linearisation}-\ref{assumption:optimalIter} hold, then $\mathbb{P}(X_s \text{ alive})$, the probability that a sample drawn from $X_s$ is alive, is lower bounded by 
 \begin{equation*}
      \Phi\left(\frac{1}{\sigma_s}\left(\frac{\mu_{s+1}}{ L_J} - \frac{1- \sqrt{\bar \alpha_{s-1}}}{\sqrt{\bar \alpha_{s-1}}}R + c_s - c_{s+1}\right) \right).
 \end{equation*}
\label{theorem:finalbound}
\end{theorem}
\begin{proof} We can find the location of the minima by solving
\begin{equation*}
    \nabla \hat J(x^\star_{s+1}, s+1) = \nabla J(x^\star_{s+1}) + \nabla b(x^\star_{s+1}, s+1) = 0.
\end{equation*}
By taking the norm, and using Assumption $\ref{assumption:linearisation}$ and \ref{assumption:LsmoothJ}
\begin{equation*}
    \frac{\mu_{s+1}}{g(x^\star_{s+1})+c_{s+1}} = ||\nabla J(x^\star_{s+1})|| \leq L_J.
\end{equation*}
By rearranging, we find
\begin{equation*}
    g(x^\star_{s+1}) \geq \frac{\mu_{s+1}}{L_J} - c_{s+1}.
\end{equation*}
Using assumption \ref{assumption:optimalIter}, we substitute $g(x_s)=g(x^\star_{s+1})$ into Lemma \ref{lemma:Pnextgivenlast}, we arrive at the final statement.
\end{proof}

Based on Theorem \ref{theorem:finalbound}, there is a simple relationship that governs the behavior of solutions in the local regime of EB-MBD. In particular, $\mu$ governs a distance from the constraint boundary. A smaller $\mu$ leads to lower cost solutions but it also leads to reduced probability of sampling live solutions. This can be counteracted by reducing  $(c_s - c_{s+1})$, which is the rate of emergence. We note that the barrier emergence rate in this late stage of the the process can be reduced by \textit{either} performing emergence earlier in the process, or by increasing number of diffusion iterations so that the emergence budget is spread over more iterations. On the other hand, if we desire fewer diffusion iterations for faster runtime, to avoid EB-MBD returning infeasible solutions -- which is possible due to dead samples with too quick barrier emergence (See experimental results in \ref{sec:2dobs}) -- we can increase $\mu$.  

\section{Experimental Results}\label{sec:results}
We implement Emerging Barrier MBD using the JAX Python package and demonstrate results for both a 2D obstacle avoidance environment, and a 3D high DOF underwater mobile manipulator system using the MuJoCo MJX as the underlying simulator\footnote{Code can be found at \color{blue}\iftoggle{anon}{link removed}{\href{https://github.com/acfr/emerging_barrier_mbd}{{https://github.com/acfr/emerging\_barrier\_mbd}}}}. MJX allows simulation rollouts to be parallelized on a GPU. All experiments were conducted on a PC with an i7-13700K, and RTX 3070 GPU and 32 GB of memory.

As the possible form of $c_s$ is a large class of functions, we parametrize $c_s$ with
\begin{equation*}
    c_s = c_{\text{max}} -c_{\text{max}}\left(\frac{s}{S}\right)^{\kappa},
\end{equation*}
to study the trade-off between early and late emergence, where $\kappa$ is a positive parameter controlling the trade-off. At $\kappa=1$, the constraint offset progresses linearly. As the effect of $\mu$ and $\kappa$ on liveliness is coupled, we keep $\mu$ fixed for all experiments to highlight its effects. 

At $\kappa >1$, the offset progresses slowly at first and faster towards the end, and vice versa for $\kappa < 1$. We use the DDPM \cite{ho_denoising_2020} noise schedule of $\beta_1 = 10^{-4}$ and $\beta_T=0.02$ with linear spacing for all experiments. $c_{\text{max}}$ is chosen to be the maximum value $g(x)$ could take on. 

\subsection{2D Obstacle Avoidance}\label{sec:2dobs}
We show the performance of EB-MBD on a simple 2D obstacle avoidance problem. The dynamics and cost are
\begin{gather*}
    \xi_{t+1} = \xi_t + 0.3\,\,\texttt{sigmoid}(0.1||u||)\frac{u_t}{||u_t||}\, ,\\
    J(\tau) = 20 \,||\xi_{T+1}-\xi_r||+\sum_t^T (||\xi_t - \xi_r||+||0.1 u_t||)\,,
\end{gather*}
where $\xi_r$ is a target position, and $\texttt{sigmoid}(x) = \frac{1}{1+e^{-x}}$. $g(\tau)$ is the signed distance to the closest obstacle encountered over the trajectory. MBD struggles due to dead samples. 

Figure \ref{fig:ebd_trajectory_comparison} shows how EB-MBD performs better than than MBD, with every trajectory reaching near the goal, while maintaining diverse solutions. In comparison, none of the MBD solutions reach the target. We notice that there are many constraint violating trajectories for MBD which is where all samples were constraint violating throughout the diffusion process.

\begin{figure}[t]
    \centering
    \subfloat{
        \includegraphics[width=0.20\textwidth]{images/Figure_Failed_Trajectory.pdf}
        \label{fig:failed_trajectory}
    }
    \subfloat{
        \includegraphics[width=0.20\textwidth]{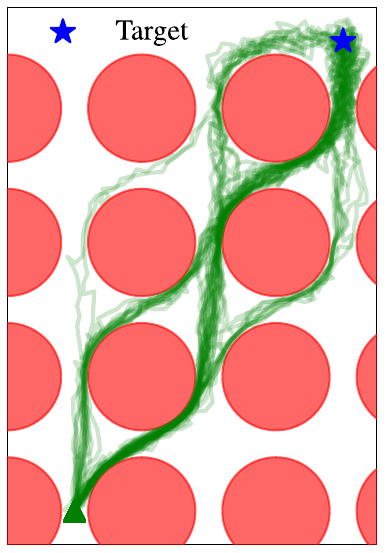}
        \label{fig:successful_trajectory}
    }
    \caption{Left: Planned trajectories are noisy for MBD due to inefficient sampling. Right: EB-MBD successfully generates trajectories from diverse high quality modes, all of which reach the vicinity of the target}
    \label{fig:ebd_trajectory_comparison}
\end{figure}

Lower $\kappa$ values result in local poor minima, similar to MBD. Increasing $\kappa$ generally shows improvement; however, the tightening constraint boundary can overshoot the current solution leading to the diffusion process dying permanently (see $\kappa > 1$ in Figure \ref{fig:constrained_violations}) as the current iterate gets stuck inside the moving constraint boundary, resulting in infeasible solutions. This can be improved by increasing $\mu$ to increase the boundary layer distance as explained in Section\ref{sec:barrieranalysis}.

\begin{figure}[t]
    \centering
    \includegraphics[width=0.9\linewidth]{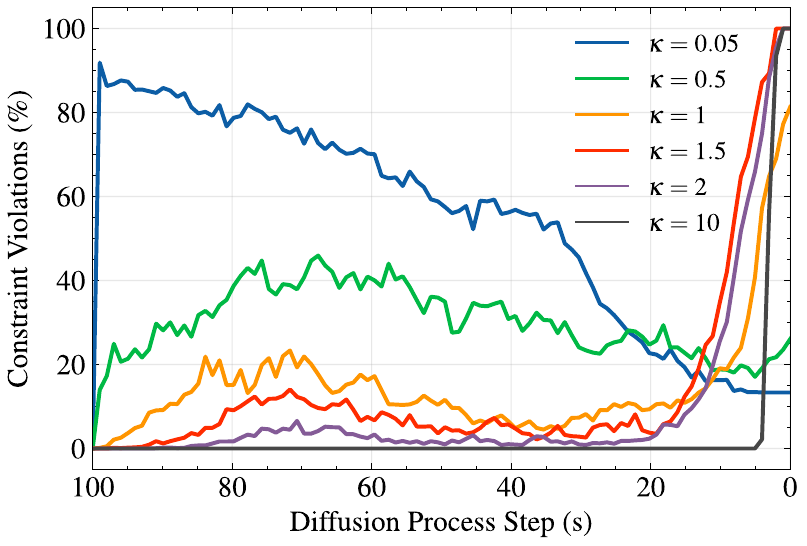}
    \caption{Percentage of samples that violate constraints over diffusion time over various $\kappa$ values. $\kappa$ being too high leads to 100\% constraint violations and infeasible outputs}
    \label{fig:constrained_violations}
\end{figure}

\subsection{Comparison to projection-based methods}
The most commonly proposed method for enforcing constraints on diffusion models and sampling-based trajectory optimization involves projections onto the constraint satisfying set performed at each iteration \cite{rastgar_priest_2024, christopher_constrained_2024, romer_diffusion_2024}. 

A projection of the trajectory $\tau_d$ onto the feasible set defined by $g(\tau)\geq0$ is denoted $\Pi_g(\tau_d)$

\begin{equation*}
    \Pi_g (\tau_d) = \arg \min_{\tau} ||\tau - \tau_d||^2 \quad \text{s.t. } g(\tau) \geq 0.
\end{equation*}
These projections are cast as NLP problems and are not uniquely defined for non-convex sets. The authors of DPCC\cite{romer_diffusion_2024} proposed using iteratively-tightening constraint for a diffusion-based MPC algorithm which is conceptually similar to progressive barriers. However, for MBD, simply ensuring constraints are satisfied after each step is not sufficient as when variance is high, most samples are still infeasible. Additionally, projections have a large computational burden as each iteration requires rollouts and taking derivatives. They also have variable runtime due to the varying convergence time of the optimizer. In comparison,  MBD and EB-MBD have effectively a constant solve time -- although the hyperparameters may need to be tuned for the system.  
\begin{figure}[t]
    \centering
    \includegraphics[width=0.85\linewidth]{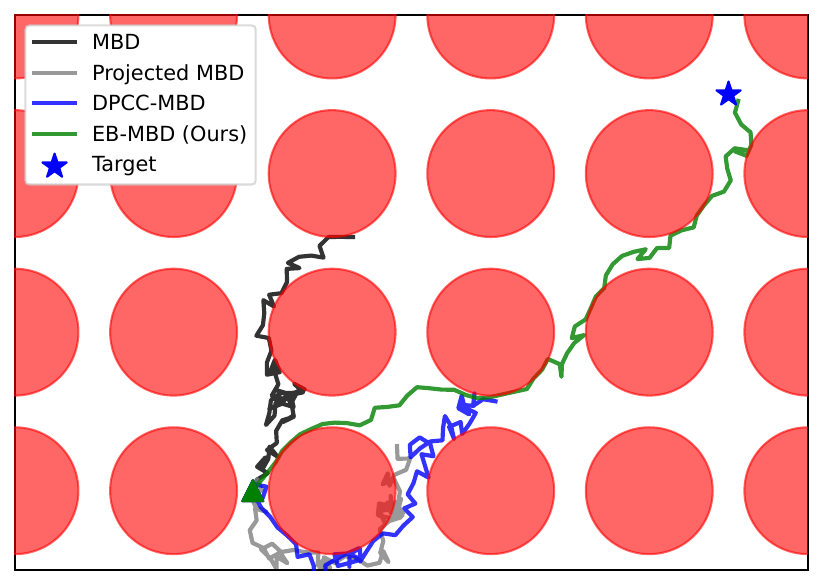}
    \caption{Typical trajectories from MBD, EB-MBD (Ours), Projected MBD\cite{christopher_constrained_2024} and DPCC-MBD\cite{romer_diffusion_2024}. EB-MBD performs substantially better than projection methods and is the only one that reaches the goal}
    \label{fig:proj_comparison}
\end{figure}
We compare against projections to the constraint set and to DPCC-style tightening constraints in Table \ref{tab:2d-perf-comparison}. We show that simply applying projection-based methods to MBD is not successful since the sampling statistics around the feasible iterate after each projection are still poor early on in the process. Projections were implemented through SciPy's SLSQP solver with analytical derivatives provided through JAX's autodifferentiation, and constraint relaxation was done similar to EB-MBD by enforcing $g(\tau)+c_s \geq 0$. 

EB-MBD is able to produce significantly lower cost trajectories and lower terminal distance to the target end point, all while taking orders of magnitude less time. Unlike the projection methods, EB-MBD also does not require taking derivatives of rollouts or the constraint function.

\begin{table}[]
\centering
\begin{tabular}{@{}llll@{}}
\toprule
\textbf{Algorithm}                                & \textbf{Mean Cost}    & \textbf{Mean Final Distance} &\textbf{Runtime (s)} \\\midrule
MBD \cite{pan_model-based_2024}                   &     514.6             &  4.3590                 &    \textbf{0.0382}  \\
Projected MBD \cite{christopher_constrained_2024} &     533.2             &  4.8517                 &    49.6752          \\
DPCC-MBD \cite{romer_diffusion_2024}              &     479.5             &  3.8221                 &    7.3621           \\ 
EB-MBD (Ours)                                     &     \textbf{234.7}    &  \textbf{0.2285}        &    \textbf{0.0383}  \\
\bottomrule \\
\end{tabular}
\caption{Comparison of EB-MBD on the 2D obstacle avoidance over 50 trajectories}
\label{tab:2d-perf-comparison}
\end{table}
\raggedbottom
\subsection{Underwater Vehicle Manipulator System}
We demonstrate EB-MBD on a high-dimensional motion planning problem in simulation for a Underwater Vehicle Manipulator System (UVMS), consisting of a BlueROV Heavy platform with a Reach Alpha manipulator, with $9$ kinematic DOF and an $11$ dimensional action space. The MuJoCo MJX simulator was used as the oracle to roll out actions and the task was to minimize a weighted combination of the quadratic cost associated with distance from the wrist of the manipulator to a target point, the actions, and the orientation of the ROV body. The target position is in a hollow box with an opening, and the constraint is to avoid the box. The cost is
\begin{equation*}
    \ell(\xi) = 10||\xi_w-r_t|| + ||\Lambda u_t||+ ||\text{Im}(q)||, \,\, \ell_T(\xi) = 100\ell(\xi, 0),
\end{equation*}
and $J(\tau)=\ell_T(\xi_{T+1}) + \sum_{t=0}^T \ell(\xi_t, u_t)$, where $\xi_w$ is the wrist position component of $\xi$, $r_t$ is the target wrist position, $\Lambda$ is a diagonal weighting matrix and $q$ is an orientation quaternion. This problem is challenging due to its high dimensionality and the complex motion required; the base is required to move and multiple joints need to coordinate together without the inertia of the robot causing a collision later in the trajectory. Figure \ref{fig:rov_traj} (a) shows how EB-MBD is able to successfully plan for the motion and \ref{fig:rov_traj} (b) shows how MBD trajectories get stuck in a poor local minima when EB-MBD finds early trajectories through the obstacle which get pushed out to better trajectories over iterations. Table \ref{tab:3d-perf-comparison} shows how EB-MBD achieves lower mean cost and higher success rates -- which we define as the percentage of time the end effector was \textit{inside} the box without collision. We were unable to run projection methods in a reasonable amount of computational time on the UVMS due to increase in complexity. We emphasize the scalability of EB-MBD as computational time only grows with the extra evaluations of $g(x)$, as opposed to the complexity of the non-linear program which often grows more rapidly for complex systems. 

\begin{figure}[t]
    \centering
        \includegraphics[width=0.9\linewidth]{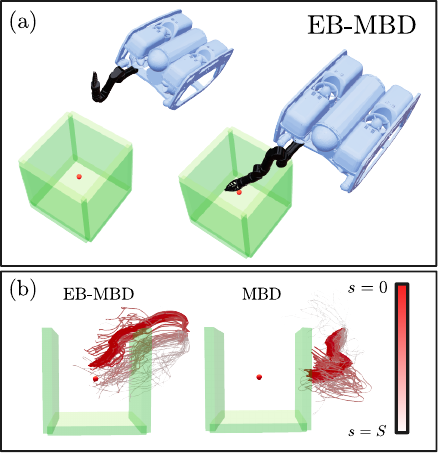}
    \caption{(a) With EB-MBD, the UVMS reaches the target point while avoiding the obstacle. (b) Intermediate end effector trajectories over a single inference of EB-MBD and MBD, with curves increasing in linewidth and redness towards the end of the process.}
    \label{fig:rov_traj}
\end{figure}

\begin{table}
\centering
\begin{tabular}{@{}llll@{}}
\toprule
\textbf{Algorithm}    & \textbf{Mean Cost} & \textbf{Success (\%)} & \textbf{Runtime (s)} \\ \midrule
MBD \cite{pan_model-based_2024}        &  427.25    &   22         &  \textbf{50}                         \\
EB-MBD (Ours)         & \textbf{362.94}         &   \textbf{48}    &  53                         \\
\bottomrule
\end{tabular}
\caption{Comparison of EB-MBD on 3D UVMS system over 50 trajectories}
\label{tab:3d-perf-comparison}
\end{table}

\section{Conclusion}
We addressed the performance degradation of model-based diffusion (MBD) for motion planning in highly constrained environments, a problem arising from poor score estimation when infeasible regions cover a large part of the solution space. We proposed Emerging-Barrier MBD, which applies an interior point-inspired time-varying barrier function, to guide solutions. This approach avoids catastrophic performance degradation and significantly improves solution quality for highly constrained problems. Our method was demonstrated on robotics collision avoidance problems, where it maintained good sample complexity and solution diversity. We analyzed the barrier schedule and its effect on solution liveliness statistics. We compare against projection-based constrained diffusion methods and show substantially better performance at orders of magnitude faster computational time.

While EB-MBD exhibited strong performance in our experiments, it does have some limitations. EB-MBD does not guarantee good solutions and good performance relies on barrier schedules well tuned for the problem. If a poor schedule is chosen, the output solutions can be entirely infeasible. Our analysis also assumes timescale separation of the diffusion process local convergence and the barrier emergence. Future work may involve adaptive barrier schedules that prevent permanently dead samples without additional barrier tuning, and further analysis of EB-MBD under milder assumptions.   
\appendix
\subsection{Half-space Integral of a Gaussian}
\label{sec:halfspace_integral}
For a Gaussian distribution with density $f_X(\cdot)$,
\begin{equation*}
    f_X(x; \mu, \Sigma) = \frac{1}{(2\pi)^{\frac{n}{2}} |\Sigma|^{\frac{1}{2}}} \exp\left( -\frac12 (x - \mu)^T \Sigma^{-1} (x - \mu)\right),
\end{equation*}
where $x, \mu \in \mathbb{R}^n, \Sigma \in \mathbb{R}^{n\times n}$. 

If $\Sigma = \sigma^2 I$, with $\sigma^2\in\mathbb{R}$, then the integral over a half-space,
\begin{equation*}
    \Omega = \{x \in \mathbb{R}^n | x^T w + b \geq 0\},
\end{equation*} can be found to be 
\begin{equation*}
    P = \int_\Omega f_X(x; \mu, \Sigma) \, dx 
    = \Phi\left(\frac{1}{\sigma ||w||}(\mu^T w + b)\right),
\end{equation*}
where $\Phi(x)$ is the 1D Gaussian CDF.

\subsection{Proof of Lemma \ref{lemma:Pnextgivenlast}} \label{appendix:lemma}
\begin{proof}
A sample is alive if the relaxed constraint $(g(x) + c_s \geq 0)$ is met. MBD samples from $X_s \sim\mathcal{N}(\frac{x_{s}}{\sqrt{\bar \alpha_{s-1}}}, \sigma_s^2 I)$. If $x_{s}$ is known, and $\varepsilon$ refers to the zero mean perturbation with the same variance as $X_s$, then we find
\begin{align*}
    & \mathbb{P}(g(X_s) + c_s \geq 0| x_{s}) \\
    = \, & \mathbb{P}\left(g\left(\frac{x_{s}}{\sqrt{\bar \alpha_{s-1}}} + \varepsilon\right) + c_s \geq 0 \, \middle| x_{s} \right) \\
    = \, & \mathbb{P}\left(g\left(\frac{x_{s}}{\sqrt{\bar \alpha_{s-1}}}\right)  - \left(\nabla g\left(\frac{x_{s}}{\sqrt{\bar \alpha_{s-1}}}\right)\right)^T\varepsilon + c_s \geq 0  \, \middle| x_{s} \right) ,
\end{align*}
where we make make use of Assumption \ref{assumption:linearisation}. Using Lipschitz continuity of $g(x)$ and using Appendix \ref{sec:halfspace_integral} to write in terms of one dimensional Gaussian cumulative distribution function,
\begin{align*}
    \geq& \, \mathbb{P} \left(g(x_{s}) - \frac{1- \sqrt{\bar \alpha_{s-1}}}{\sqrt{\bar \alpha_{s-1}}}||x_{s}|| \right. - 
    \\ &\qquad\qquad\qquad \left. \nabla g\left(\frac{x_{s}}{\sqrt{\bar \alpha_{s-1}}}\right)^T\varepsilon +  c_s \geq 0 \middle| x_{s} \right) \\
    = &\, \Phi\left(\frac{1}{\sigma_s}\left(g(x_{s}) - \frac{1 - \sqrt{\bar \alpha_{s-1}}}{\sqrt{\bar \alpha_{s-1}}}||x_{s}||  + c_s\right) \right).
\end{align*}
\end{proof}

\bibliography{references}  

\end{document}